\documentclass{article}
\usepackage[round]{natbib}   
\usepackage[english]{babel}
\usepackage[utf8x]{inputenc}
\usepackage[T1]{fontenc}
\usepackage[top=3cm,bottom=2cm,left=3cm,right=3cm,marginparwidth=1.75cm]{geometry}

\usepackage{amsmath, amsfonts, amsthm}
\usepackage{graphicx}
\usepackage[colorinlistoftodos]{todonotes}
\usepackage[colorlinks=true, allcolors=blue]{hyperref}
\usepackage{algorithm}
\usepackage[noend]{algpseudocode}
\usepackage{float}
\newfloat{algorithm}{t}{lop}

\newcommand{\remove}[1]{}

\def\calx{{\cal X}}
\def\caly{{\cal Y}}
\def\calz{{\cal Z}}
\def\cald{{\cal D}}
\def\calc{{\cal C}}
\def\cents{\hbox{\rm\rlap/c}}

\def\reals{\mathbb{R}}

\def\emp{\hat{\ell}}
\def\err{\ell}

\def\sgn{\mathrm{sgn}}
\def\I{\mathbb{I}}
\def\FP{\mathrm{FP}}
\def\FN{\mathrm{FN}}

\def\FNR{\mathrm{FNR}}

\DeclareMathOperator*{\E}{E}
\DeclareMathOperator*{\argmin}{arg\,min}

\newtheorem{theorem}{Theorem}
\newtheorem{definition}{Definition}

\newtheorem{observation}{Observation}
\newtheorem{lemma}{Lemma}

\title{Decoupled classifiers for fair and efficient machine learning}
\author{Cynthia Dwork, Nicole Immorlica, Adam Tauman Kalai, and Max Leiserson}

\begin{document}
\maketitle

\begin{abstract}
When it is ethical and legal to use a sensitive attribute (such as gender or race) in machine learning systems, the question remains how to do so. We show that the naive application of machine learning algorithms using sensitive attributes leads to an inherent tradeoff in accuracy between groups. We provide a simple and efficient {\em decoupling} technique, that can be added on top of any black-box machine learning algorithm, to learn different classifiers for different groups. Transfer learning is used to mitigate the problem of having too little data on any one group.

The method can apply to a range of fairness criteria. In particular, we require the application designer to specify as joint loss function that makes explicit the trade-off between fairness and accuracy. Our reduction is shown to efficiently find the global optimum loss as long as the objective has a certain natural {\em monotonicity} property. Monotonicity may be of independent interest in the study of fairness in algorithms.
\end{abstract}

\section{Introduction}
As algorithms are increasingly used to make decisions of social consequence, the social values encoded in these decision-making procedures are the subject of increasing study, with fairness being a chief concern~\citep{PedreschiRT08,ZliobateKC11,KamishmimaAS11,DHPRZ2012,FriedlerSV2016,Propublica2016,C2017,KMR2016,HPS2016,KusnerLRS2017,Berk}.
{\em Classification and regression algorithms} are one particular locus of fairness concerns.  Classifiers map individuals to outcomes: applicants to accept/reject/waitlist; adults to credit scores; web users to advertisements; felons to estimated recidivism risk. Informally, the concern is whether individuals are treated ``fairly,'' however this is defined.  Still speaking informally, there are many sources of unfairness, prominent among these being training the classifier on historically biased data and a paucity of data for under-represented groups leading to poor performance on these groups, which in turn can lead to higher risk for those, such as lenders, making decisions based on classification outcomes.    

\begin{figure}
  \centering \includegraphics[height=1.3in]{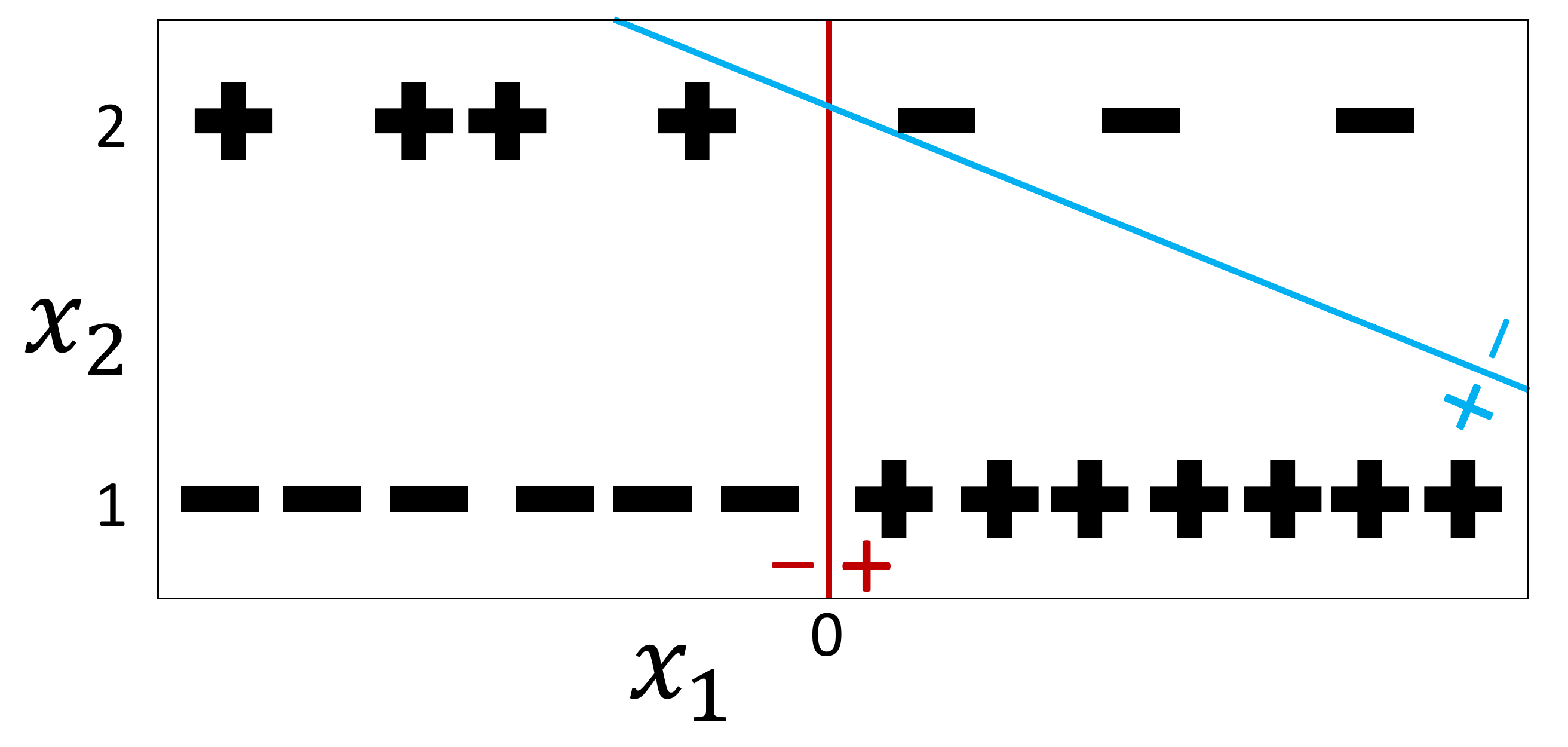}
  \caption{Disregarding group membership (feature $x_2$), the most accurate linear classifier (red) perfectly classifies the majority class but perfectly {\em misclassifies} the minority group. No single linear classifier can achieve greater than 50\% (i.e., random) accuracy simultaneously on both groups.}
  \label{fig:scary}
\end{figure}

Should ML systems use sensitive attributes, such as gender or race if available? The legal and ethical factors behind such a decision vary by time, country, jurisdiction, and culture, and downstream application. Still speaking informally, it is known that ``ignoring'' these attributes does not ensure fairness, both because they may be closely correlated with other features in the data and because they provide context for understanding the rest of the data, permitting a classifier to incorporate information about cultural differences between groups~\citep{DHPRZ2012}.  Using sensitive attributes may increase accuracy for all groups and may avoid biases where a classifier favors members of a minority group that meet criteria optimized for a majority group, as illustrated visually in Figure \ref{fig:image-search} of Section \ref{sec:imageexperiment}.


In this paper, we consider {\em how} to use a sensitive attribute such as gender or race to maximize fairness and accuracy, assuming that it is legal and ethical. If a data scientist wanted to fit, say, a simple linear classifier, they may use the raw data, upweight/oversample data from minority groups, or employ advanced approaches to fitting linear classifiers that aim to be accurate and fair. No matter what they do and what fairness criteria they use, assuming no linear classifier is perfect, they may be faced with an inherent tradeoff between accuracy on one group and accuracy on another. As an extreme illustrative example, consider the two group setting illustrated in Figure \ref{fig:scary}, where feature $x_1$ perfectly predicts the binary outcome $y \in \{-1,1\}$. For people in group 1 (where $x_2=1$), the majority group, $y=\sgn(x_1)$, i.e., $y=1$ when $x_1>0$ and $-1$ otherwise. However, for the minority group where $x_2=2$, exactly the opposite holds: $y=-\sgn(x_1)$. Now, if one performed classification without the sensitive attribute $x_2$, the most accurate classifier predicts $y=\sgn(x_1)$, so the majority group would be perfectly classified and the minority group would be classified as inaccurately as possible. However, even using the group membership attribute $x_2$, it is impossible to simultaneously achieve better than $50\%$ (random) accuracy on both groups. This is due to limitations of a linear classifier $\sgn(w_1 x_1 + w_2 x_2 + b)$, since the same $w_1$ is used across groups. 

Put another way, if a single linear classifier is used, members of one of the groups may say, ``why don't you use this other classifier on our group, which is more accurate for us but still classifies the same number of us as positive (and hence does not change our status with respect to other groups)?'' Now, linear classifiers are of interest since they are common (e.g., the output of linear regression or SVMs). However, the data scientist's predicament is not specific to linear classifiers. For example, as we show in Theorem \ref{thm:tradeoff}, decision trees of bounded size are also subject to this inherent tradeoff of accuracy across groups, even though decision trees can branch on the sensitive bit. 

In this paper we define and explore {\em decoupled} classification systems, in which a separate classifier is trained on each group.  Training a classifier involves minimizing a loss function that penalizes errors; examples include mean squared loss and absolute loss.  In decoupled classification systems one first obtains, for each group separately, a collection of classifiers differing in the numbers of {\em positive} classifications returned for the members of the given group.  Let this set of outputs for group~$k$ be denoted $C_k$, $k = 1,\ldots,K$. The output of the decoupled training step is an element of $C_1 \times \ldots\times C_K$, that is, a single classifier for each group.  The output is chosen to minimize a {\em joint loss function} that can penalize differences in classification statistics between groups.  Thus the loss function can capture {\em group fairness} properties relating the treatment of different groups, {\it e.g.}, the false positive (respectively, false negative) rates are the same across groups; the demographics of the group of individuals receiving positive (negative) classification are the same as the demographics of the underlying population; the positive predictive value is the same across groups.\footnote{In contrast {\em individual fairness}~\cite{DHPRZ2012} requires that {\em similar people are treated similarly}, which requires a task-specific, culturally-aware, similarity metric.} By pinning down a specific objective, the modeler is forced to make explicit the tradeoff between accuracy and fairness, since often both cannot simultaneously be achieved. Finally, a generalization argument relates fairness properties, captured by the joint loss on the training set, to similar fairness properties on the distribution from which the data were drawn.  We broaden our results so as to enable the use of {\em transfer learning} to ameliorate the problems of low data volume for minority groups.

\begin{figure}
  \centering \includegraphics[height=1.4in]{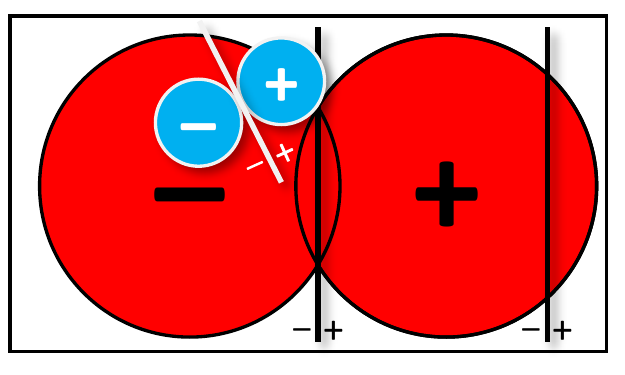}
  \caption{Decoupling helps both majority (red) and minority (blue) groups each maximize accuracy from different linear classifiers (white line and left black line). If, say, equal numbers of positives are required from both groups, the white line and right black line would maximize average accuracy.  }
  \label{fig:venn}
\end{figure}


The following observation provides a property essential for efficient decoupling.  A {\em profile} is a vector specifying, for each group, a number of positively classified examples from the training set. 
For a given profile $(p_1,\ldots, p_K)$, the most accurate classifier also simultaneously minimizes the false positives and false negatives.  {\em It is the choice of profile that is determined by the joint loss criterion.}  We show that, as long as the joint loss function satisfies a weak form of {\em monotonicity}, one can use off-the-shelf classifiers to find a decoupled solution that minimizes joint loss.  

The monotonicity requirement is that the joint loss is non-decreasing in error rates, for any fixed profile. This sheds some light on the thought-provoking impossibility results of \cite{C2017} and \cite{KMR2016} on the impossibility of simultaneously achieving three specific notions of group fairness (see Observation \ref{obs:kevefef} in Section \ref{sec:fairnessDiscussion}).

Finally, we present experiments on 47 datasets downloaded from {\tt http://openml.org}. The experiments are ``semi-synthetic'' in the sense that the first binary feature was used as a substitute sensitive feature since we did not have access to sensitive features. We find that on many data sets our algorithm improves performance while much less often decreasing performance. 

\smallskip
\noindent
{\bf Remark.}
The question of whether or not to use decoupled classifiers is orthogonal to our work, which explores the mathematics of the approach, and a comprehensive treatment of the pros and cons is beyond our expertise.  Most importantly, we emphasize that decoupling, together with a ``poor'' choice of joint loss, could be used unfairly for discriminative purposes.  Furthermore, in some jurisdictions using a different classification method, or even using different weights on attributes for members of demographic groups differing in a protected attribute, is illegal for certain classification tasks, {\it e.g.} hiring.  Even baring legal restrictions, the assumption that group membership is an input bit is an oversimplification, and in reality the information may be obscured, and the definition of the groups may be ambiguous at best.  Logically pursuing the idea behind the approach it is not clear which intersectionalities to consider, or how far to subdivide.  Nonetheless, we believe decoupling is valuable and applicable in certain settings and thus merits investigation.

\smallskip
The contributions of this work are: (a) showing how, when using sensitive attributes, the straightforward application of many machine learning algorithms will face inherent tradeoffs between accuracy across different groups, (b) introducing an efficient decoupling procedure that outputs separate classifiers for each class using transfer learning, (c) modeling fair and accurate learning as a problem of minimizing a joint loss function, and (d) presenting experimental results showing the applicability and potential benefit of our approach.

\subsection{Related Work}


Group fairness has a variety of definitions, including conditions of {\em statistical parity}, {\em class balance} and {\em calibration}.  In contrast to individual fairness, these conditions constrain, in various ways, the dependence of the classifier on the sensitive attributes.  The statistical parity condition requires that the assigned label of an individual is independent of sensitive attributes.  The condition formalizes the legal doctrine of disparate impact imposed by the Supreme Court in Griggs v Duke Power Company.  Statistical parity can be approximated by either modifying the data set or by designing classifiers subject to fairness regularizers that penalize violations of statistical parity (see~\cite{FFMSV2015} and references therein).  \cite{DHPRZ2012} propose a ``fair affirmative action'' methodology  that carefully relaxes between-group individual fairness constraints in order to achieve group fairness. \cite{ZWSPD2013} introduce a representational approach that attempts to ``forget'' group membership while maintaining enough information to classify similar individuals similarly; this approach also permits generalization to unseen data points.  To our knowledge, the earliest work on trying to learn fair classifiers from historically biased data is by~\cite{PedreschiRT08}; see also~\citep{ZliobateKC11} and~\citep{KamishmimaAS11}.  

The class-balanced condition (called {\em error-rate balance} by~\cite{C2017} or {\em equalized odds} by~\cite{HPS2016}), similar to statistical parity, requires that the assigned label is independent of sensitive attributes, but only {\em conditional on the true classification of the individual}. For binary classification tasks, a class-balanced classifier results in equal false positive and false negative rates across groups. One can also modify a given classifier to be class-balanced while minimizing loss by adding label noise \citep{HPS2016}. 

The well-calibrated condition requires that, conditional on their label, an equal fraction of individuals from each group have the same true classification.  A well-calibrated classifier labels individuals from different groups with equal accuracy. The class-balanced solution \citep{HPS2016} also fails to be well-calibrated.  \cite{C2017} and \cite{KMR2016} concurrently showed that, except in cases of perfect predictions or equal base rates of true classifications across groups, there is no class-balanced and well-calibrated classifier.

A number of recent works explore causal approaches to defining and detecting (un)fairness~\citep{nabi2017fair,kusner2017counterfactual,bareinboim2016causal,kilbertus2017avoiding}. See the beautiful primer of \cite{pearl2016causal} for an introduction to the central concepts and machinery. 

Finally, we mention that sensitive attributes are used in various real-world systems. As one example, \cite{Romm2017} describe using such features in an admissions matching system for masters students in Israel.

\section{Preliminaries}

Let $\calx = \calx_1 \cup \calx_2 \cup \ldots \cup \calx_K$ be the set of possible {\em examples} partitioned by group. 
The set of possible {\em labels} is $\caly$ and the set of possible {\em classifications} is $\calz$. A {\em classifier} is a function $c:\calx \rightarrow \calz$. We assume that there is a fixed family $\calc$ of classifiers.

We suppose that there is a joint distribution $\cald$ over labeled examples $x, y \in \calx \times \caly$ and we have access to $n$ training examples $(x_1,y_1), \ldots, (x_n, y_n) \in \calx\times \caly$ drawn independently from $\cald$. We denote by $g(x)$ the group number to which $x$ belongs and $g_i=g(x_i)$, so $x_i \in \calx_{g_i}$. 

Finally, as is common, we consider the loss $\ell_\cald(c)=\E_{x,y \sim \cald}[\ell(y,c(x))]$ for an application-specific loss function $\ell:\caly\times\calz \rightarrow \reals$ where $\ell(y,z)$ accounts for the cost of classifying as $z$ an example whose true label is $y$. The  group-$k$ loss for $\cald,c$ is defined to be $\ell_{\cald k}(c)=\E_\cald[\ell(y, c(x))|x \in \calx_k]$ or 0 if $\cald$ assigns 0 probability to $\calx_k$. The standard approach in ML is to minimize $\ell_\cald(c)$ over $c \in \calc$. Common loss functions include the $L_1$ loss $\ell(y,z)=|y-z|$ and $L_2$ loss $\ell(y,z)=(y-z)^2$. In Section \ref{sec:loss}, we provide a methodology for incorporating a range of fairness notions into loss. 

\smallskip
\noindent
{\bf Notation}. The indicator function $\I[\phi]$ is 1 if $\phi$ is true and 0 otherwise. Let $\mathbb{N}=\{0,1,2,\ldots\}$ 
and $2^S$ denote the set of subsets of set $S$. The inner product of vectors $x, w\in \reals^d$ is written $w\cdot x$. For set $S$, $S^*=S^0 \cup S^1 \cup S^2 \cup \ldots$ is the set of finite sequences of elements of $S$. 

\section{Decoupling and the cost of coupling}\label{sec:tradeoff}

For a vector of $K$ classifiers, $\vec{c}=(c_1, c_2, \ldots, c_K)$, the decoupled classifier $\gamma_{\vec{c}}:\calx\rightarrow\calz$ is defined to be $\gamma_{\vec{c}}=c_{g(x)}(x)$.
The set of decoupled classifiers is denoted $\gamma(\calc)=\{\gamma_{\vec{c}}~|~\vec{c} \in \calc^K\}.$ Some classifiers, such as decision trees of unbounded size over $\calx = \{0,1\}^d$, are already decoupled, i.e.,  $\gamma(\calc)=\calc$. As we shall see, however, in high dimensions common families of classifiers in use are coupled to avoid the curse of dimensionality. 

The cost of coupling of a family $\calc$ of classifiers (with respect to $\ell$) is defined to be the worst-case maximum of the difference between the loss of the most accurate coupled and decoupled classifiers over distributions $\cald$.
$$\text{cost-of-coupling}(\calc, \ell) = \max_{\cald \in \Delta(\calx\times\caly)} \left[\min_{c \in \calc} \ell_\cald(c)-\min_{\gamma_{\vec{c}} \in \gamma(\calc)} \ell_\cald(\gamma_{\vec{c}})\right].$$
Here $\Delta(S)$ denotes the set of probability distributions over set $S$. To circumvent measure-theoretic nuisances, we require $\calc, \calx, \caly$ to be finite sets. Note that numbers on digital computers are all represented using a fixed-precision (bounded number of bits) representation, and hence all these sets may be assumed to be of finite (but possibly exponentially large) size.

We now show that the cost of coupling related to fairness across groups. 
\begin{lemma}
Suppose $\text{cost-of-coupling}(\calc, \ell)=\cents$. Then there is a distribution $\cald$ such that no matter which classifier $c \in \calc$ is used, there will always be a group $k$ and a classifier $c' \in \calc$ whose group-$k$ loss is at least $\cents$ smaller than that of $c$, i.e., $\ell_{\cald k}(c')\leq \ell_{\cald k}(c)-\cents$.
\end{lemma}
\begin{proof}
Let $\gamma_{\vec{c}'}$ be a decoupled classifier with minimal loss where $\vec{c}'=(c_1', \ldots, c_K')$. This loss is a weighted average (weighted by demography) of the average loss on each group. Hence, for any $c$, there must be some group $k$ on which the loss of $c_k'$ is $\cents$ less than that of $c$. 
\end{proof}
Hence, if the cost of coupling is positive, then  the learning algorithm that selects a classifier faces an inherent tradeoff in accuracy across groups. We now show that the cost of coupling is large (a constant) for linear classifiers and decision trees.
\begin{theorem}\label{thm:tradeoff}
Fix $\calx = \{0,1\}^d$, $\caly=\{0,1\}$, and $K=2$ groups (encoded by the last bit of $x$). Then the cost of coupling is at least 1/4 for:
\begin{enumerate}
\item {\bf Linear regression}: $\calz=\mathbb{R}$, $\calc=\{w \cdot x + b~|~w \in \reals^d, b \in \reals\}$, and $\ell(y,z)=(y-z)^2$
\item {\bf Linear separators}: $\calz=\{0,1\}$, $\calc=\{\I[w \cdot x + b \geq 0]~|~w \in \reals^d, b \in \reals\}$, and $\ell(y,z)=|y-z|$
\item {\bf Bounded-size decision trees}: For $\calz=\{0,1\}$, $\calc$ being the set of binary decision trees of size $\leq 2^s$ leaves, and $\ell(y, z)=|y-z|$
\end{enumerate}
\end{theorem}
All further proofs are deferred to the Appendix. We note that it is straightforward to extend the above theorem to generalized linear models, i.e., functions $c(x)=u(w \cdot x)$ for monotonic functions $u: \reals\rightarrow\reals$, which includes logistic regression as one common special case. It is also possible, though more complex, to provide a lower bound on the cost of coupling of neural networks, regression forests, or other complex families of functions of bounded representation size $s$. In order to do so, one needs to simply show that the size-$s$ functions are sufficiently rich in that there are two different size-$s$ classifiers $\vec{c}=(c_1,c_2)$ such that $\gamma_{\vec{c}}$ has 0 loss (say over the uniform distribution on $\calx$) but that every single size-$s$ classifier has significant loss.

\section{Joint loss and monotonicity}\label{sec:loss}

As discussed, the classifications output by an ML classifier are often evaluated by their empirical loss $\frac{1}{n}\sum_i \ell(y_i, z_i)$.
To account for fairness, we generalize loss to joint classifications across groups. In particular, we consider an application-specific {\em joint loss} $\hat{L}:([K]\times \caly\times\calz)^*\rightarrow\reals$ that assigns a cost to a set of classifications, where $[K]=\{1,2,\ldots,K\}$ indicates the group number for each example. A joint loss might be, for parameter $\lambda \in [0,1]$: 
$$\hat{L}\bigl(\langle g_i,y_i,z_i\rangle_{i=1}^n\bigr)=\frac{\lambda}{n}\sum_{i=1}^n |y_i-z_i| + \frac{1-\lambda}{n} \sum_{k=1}^K
\left|\sum_{i:g_i=k} z_i - \frac{1}{K}\sum_i z_i \right|.$$
The above $\hat{L}$ trades off accuracy for differences in number of positive classifications across groups. For $\lambda=1$, this is simply $L_1$ loss, while for $\lambda=0$, the best classifications would have an equal number of positives in each group.
Joint loss differs from a standard ML loss function in two ways. 
First, joint loss is aware of the sensitive group membership. Second, it depends on the complete labelings and is not simply a sum over labels. Even with only $K=1$ group, this captures situations beyond what is representable by the sum $\sum \ell(y_i,z_i)$. A simple example is when one seeks exactly $P$ positive examples:
$$\hat{L}\bigl(\langle g_i,y_i,z_i\rangle_{i=1}^n\bigr)=\begin{cases}\frac{1}{n}\sum |y_i-z_i|&\text{if } \sum z_i=P\\
1 &\text{otherwise}.
\end{cases}
$$
Since $\frac{1}{n}\sum |y_i-z_i|\leq 1$, the $1$ ensures that the loss minimizer will have exactly $P$ positives, if such a classifier exists in $\calc$ for the data. 
In this section, we denote joint loss $\hat{L}$ with the hat notation indicating that it is an empirical approximation. In the next section we will define joint loss $L$ for distributions. We denote classifications by $z_i$ rather than the standard notation $\hat{y}_i$ which suggests predictions, because, as in the above loss, one may choose classifications $z\neq y$ even with perfect knowledge of the true labels.

For the remainder of our analysis, we henceforth consider binary labels and classifications, $\caly=\calz=\{0,1\}$. Our approach is general, however, and our experiments include regression. For a given $\langle x_i, y_i, z_i\rangle_{i=1}^n$, and for any group $k\leq K$ and all $(y,z) \in \{0,1\}^2$, recall that the groups are $g_i=g(x_i)$ and define:
\begin{align*}
\text{{\em counts}:\ \ } n_k&=\bigl|\{i~|~g_i=k\}\bigr| \in \{1,2,\ldots, n\}\\
\text{{\em profile}:\ \ } \hat{p}_k&=\frac{1}{n}\sum_{i: g_i=k} z_i \in \left[0, n_k/n\right]\\
\text{{\em group losses}:\ \ } \emp_k&=\frac{1}{n_k}\sum_{i: g_i=k} |z_i-y_i| \in [0,1]
\end{align*}
Note that the normalization is such that the standard 0-1 loss is $\sum_k \frac{n_k}{n} \emp_k$ and the fraction of positives within any class is $\frac{n}{n_k}\hat{p}_k$.


\begin{table}
\centering
\begin{tabular}{rl}
Balanced loss:&
$\displaystyle \hat{L}_{B}=\frac{1}{K}\sum_k \emp_k$\vspace{0.1in}\\
$L_1$ loss:&
$\displaystyle \hat{L}_1=\frac{1}{n}\sum_i |y_i-z_i| =\sum_k\frac{n_k}{n}\emp_k$\vspace{0.1in}\\
Strict numerical parity:&$\displaystyle \hat{L}_{NP}=\begin{cases}\hat{L}_1&\text{if }\hat{p}_1=\hat{p}_2=\ldots=\hat{p}_K\\
1& \text{otherwise}
\end{cases}$\vspace{0.1in}\\
Numerical parity:&$\displaystyle \hat{L}_{NP\lambda}=\lambda \hat{L}_1 +(1-\lambda)\sum_k \left|\hat{p}_k-\frac{1}{K}\sum_{k'}\hat{p}_{k'}\right|$\vspace{0.1in}\\
Strict demographic parity:&$\displaystyle \hat{L}_{DP}=\begin{cases}\hat{L}_1&\text{if }\hat{p}_1\frac{n}{n_1}=\hat{p}_2\frac{n}{n_2}=\ldots=\hat{p}_K\frac{n}{n_K}\\
1& \text{otherwise}
\end{cases}$\vspace{0.1in}\\
Demographic parity:&$\displaystyle \hat{L}_{DP\lambda}=\lambda \hat{L}_1 +(1-\lambda)\sum_k \left|\hat{p}_k\frac{n}{n_k}-\frac{1}{K}\sum_{k'}\hat{p}_{k'}\frac{n}{n_{k'}}\right|$\vspace{0.1in}\\
Fixed profile:&$\displaystyle \hat{L}_{\vec{p}^*}=\begin{cases}\hat{L}_1&\text{if }\hat{p}_1=p_1^* \wedge \ldots\wedge  \hat{p}_K=p_K^*\\
1&\text{otherwise}\end{cases}$\vspace{0.1in}\\
False-negative-rate parity:&$\displaystyle \hat{L}_{\FNR\lambda}=\lambda \hat{L}_1 + (1-\lambda)\sum_k\left|\FNR_k-\frac{1}{K}\sum_{k'} \FNR_{k'}\right|$\vspace{0.1in}
\end{tabular}
\caption{A number of different unfairness functions in terms of profile $\hat{p}_k$ and group losses $\emp_k$. Note that $\FNR_k$ can be derived from these quantities by Equation (\ref{eq:fn}). For $\lambda$ close to 0, the minimizer of $\hat{L}_{NP\lambda}$ should have a nearly uniform profile, the minimizer of $\hat{L}_{DP\lambda}$ should have a profile proportional to $\hat{p}_k\propto n_k$, and the minimizer of $\hat{L}_{\FNR\lambda}$ should have equal false negative rates. All losses are monotonic except $\hat{L}_{\FNR\lambda}$ (see Observation \ref{obs:kevefef}).
}\label{tab:losses}
\end{table}

Table \ref{tab:losses} illustrates some natural joint losses. In many applications there is a different cost for false positives where $(y,z)=(0,1)$ and false negatives where $(y,z)=(1,0)$. The fractions of false positives and negatives are defined, below, for each group $k$. They can be computed based on the fraction of positive labels in each group $\pi_k$:
\begin{align}
\pi_k &= \frac{1}{n_k}\sum_{i:g_i=k} y_i \nonumber\\
\FP_k&=\frac{1}{n_k}\sum_{i:g_i=k} z_i(1-y_i) =\frac{\emp_k + \hat{p}_k\frac{n}{n_k} - \pi_k}{2}\label{eq:fp}\\
\FN_k&=\frac{1}{n_k}\sum_{i:g_i=k} (1-z_i)y_i =\frac{\emp_k + \pi_k - \hat{p}_k\frac{n}{n_k}}{2}\label{eq:fn},
\end{align}
While minimizing group loss $\emp_k = \FP_k + \FN_k$ in general does not minimize false positives or false negatives on their own, the above implies that for a fixed profile $\hat{p}_k$, the most accurate classifier on group $k$ simultaneously minimizes false positives and false negatives. The above can be derived by adding or subtracting the equations $\emp_k = \FP_k+\FN_k$ (since every error is a false positive or a false negative) and $\frac{n}{n_k}\hat{p}_k = \FP_k + (\pi_k - \FN_k)$ (since every positive classification is either a false positive or true positive, and the fraction of true positives from group $k$ are $\pi_k - \FN_k$). 
We also define the {\em false negative rate} $\FNR_k=\FN_k/\pi_k$. False positive rates can be defined similarly. 

Equations (\ref{eq:fp}) and (\ref{eq:fn}) imply that, if one desires fewer false positives and false negatives (all other things being fixed), then greater accuracy is better. That is, for a fixed profile, the most accurate classifier simultaneously minimizes false positives and false negatives. This motivates the following monotonicity notion.
\begin{definition}[Monotinicity]\label{def:monotonicity1}
Joint loss $\hat{L}$ is monotonic if, for any fixed $\langle g_i,y_i\rangle_{i=1}^n\in ([K] \times \caly)^*$, $\hat{L}$ can be written as $c(\langle\emp_k, \hat{p}_k\rangle_{k=1}^K)$ where 
$c:[0,1]^{2K}\rightarrow\reals$ is a function that is nondecreasing in each $\emp_k$ fixing all other inputs to $c$.
\end{definition}
That is, for a fixed profile, increasing $\emp_k$ can only increase joint loss. To give further intuition behind monotonicity, we give two other equivalent definitions.  
\begin{definition}[Monotonicity]\label{def:monotonicity2}
Joint loss $\hat{L}$ is monotonic if, for any $\langle g_i,y_i,z_i\rangle_{i=1}^n \in ([K] \times \caly \times \calz)^*$, and any $i,j$ where $g_i=g_j$, $y_i\leq y_j$ and $z_i\leq z_j$: swapping $z_i$ and $z_j$ can only increase loss, i.e.,
$$\hat{L}(\langle g_i,y_i,z_i\rangle_{i=1}^n) \leq \hat{L}(\langle g_i,y_i,z'_i\rangle_{i=1}^n),$$
where $z'$ is the same as $z$ except $z'_i=z_j$ and $z'_j=z_i$.
\end{definition}
We can see that if $y_i=y_j$ then swapping $z_i$ and $z_j$ does not change the loss (because the condition can be used in either order). This means that the loss is ``semi-anonymous'' in the sense that it only depends on the numbers of true and false positives and negatives for each group. The more interesting case is when $(y_i, y_j)=(0,1)$ where it states that the loss when $(z_i, z_j)=(0,1)$ is no greater than the loss when  $(z_i, z_j)=(1,0)$. 
Finally, monotonicity can also be defined in terms of false positives and false negatives.
\begin{definition}[Monotonicity]\label{def:monotonicity3}
Joint loss $\hat{L}$ is monotonic if, for any $\langle g_i,y_i,z_i\rangle_{i=1}^n \in ([K] \times \caly \times \calz)^*$, and any alternative classifications $z'_1, \ldots, z'_n$ such that, in each group $k$, the same profile as $z$ but all smaller or equal false positive rates {\em and} all smaller or equal false negative rates, the loss of classifications $z'_i$ is no greater than that of $z_i$.
\end{definition}
\begin{lemma}\label{lem:monotonicity}
Definitions \ref{def:monotonicity1}, \ref{def:monotonicity2}, and \ref{def:monotonicity3} of Monotonicity are equivalent.  
\end{lemma}

All the losses in Table \ref{tab:losses} except the last can be seen to be monotonic by Definition \ref{def:monotonicity1}. One may be tempted to consider a simpler notion of monotonicity, such as requiring the loss with $z_i=y_i$ to be no greater than that of $z_i=1-y_i$, fixing everything else. However, this would rule out many natural monotonic joint losses $\hat{L}$ that would be ruled out by such a strong assumption, such as $\hat{L}_{DP}, \hat{L}_{DP\lambda}, \hat{L}_{NP\lambda}, \hat{L}_{\vec{p}^*}$, and $\hat{L}_{\FNR\lambda}$ from Table \ref{tab:losses}.

\subsection{Discussion: fairness metrics versus objectives}\label{sec:fairnessDiscussion}
As can be seen from Table \ref{tab:losses}, the monotonicity requirement admits a range of different fairness criteria, but not all. We do not mean to imply that monotonicity is necessary for fairness, but rather to discuss the implications of minimizing a non-monotonic loss objective.  The following example helps illustrate the boundary between monotonic and non-monotonic.
\begin{observation}\label{obs:kevefef}
Fix $K=2$. The following joint loss is monotonic if and only if $\lambda\leq 1/2$:
$$(1-\lambda)(\emp_1+\emp_2) + \lambda |\emp_1-\emp_2|.$$
\end{observation}
The loss in the above lemma trades off accuracy for differences in loss rates between groups. What we see is that monotonic losses can account, to a limited extent, for differences across groups in fractions of errors, and related statements can be made for combinations of rates of false positive and false negative, inspired by ``equal odds'' definitions of fairness. However, when the weight $\lambda$ on the fairness term exceeds 1/2, then the loss is non-monotonic and one encounters situations where one group is punished with lower accuracy in the name of fairness. This may still be desirable in a context where equal odds is a primary requirement, and one would rather have random classifications (e.g., a lottery) than introduce any inequity.   

What is the goal of an objective function? We argue that a good objective function is one whose optimization leads to favorable outcomes, and should not be confused with a fairness metric whose goal is {\em quantify} unfairness. Often, a different function is appropriate for quantifying unfairness than for optimizing it. For example, the difference in classroom performance across groups may serve as a good metric of unfairness, but it may not be a good objective on its own. The root cause of the unfairness may have begun long before the class. Now, suppose that the objective from the above observation was used by a teacher to design a semester-long curriculum with the best intention of increasing the minority group's performance to the level of the majority. If there is no curriculum that in one semester increases one group's performance to the level of another group's performance, then optimizing the above loss for $\lambda > 1/2$ leads to an undesirable outcome: the curriculum would be chosen so as to intentionally {\em mis}teaching students the higher-performing group of students so that their loss increases to match that of the other group. This can be see by rewriting the loss as follows:
$$(1-\lambda)(\emp_1+\emp_2) + \lambda |\emp_1-\emp_2|=2\lambda \max\{\emp_1,\emp_2\}+(1-2\lambda)(\emp_1+\emp_2).$$
This rewriting illuminates why $\lambda \leq 1/2$ is necessary for monotonicity, otherwise there is a negative weight on the total loss. $\lambda=1/2$ corresponds to maximizing the minimum performance across groups while $\lambda=0$ means teaching to the average, and $\lambda$ in between allows interpolation. However, putting too much weight on fairness leads to undesirable punishing behavior.

\section{Minimizing joint loss on training data}

Here, we show how to use learning algorithm to find a decoupled classifier in $\gamma(\calc)$ that is optimal on the training data. In the next section, we show how to generalize this to imperfect randomized classifiers that generalize to examples drawn from the same distribution, potentially using an arbitrary transfer learning algorithm.

Our approach to decoupling uses a learning algorithm for $\calc$ as a black box.
A $\calc${\em -learning algorithm} $A:(\calx \times \caly)^* \rightarrow 2^\calc$ returns one or more classifiers from $\calc$ with differing numbers of positive classifications on the training data, i.e., for any two distinct $c, c' \in A\bigl(\langle x_i,y_i\rangle_{i=1}^n)$, $\sum_i c(x_i)\neq \sum_i c'(x_i)$. In ML, it is common to simultaneously output classifiers with varying number of positive classifications, e.g., in computing ROC or precision-recall curves \citep{davis2006relationship}. Also note that a classifier that purely minimizes errors can be massaged into one that outputs different fractions of positive and negative examples by reweighting (or subsampling) the positive- and negative-labeled examples with different weights. 

Our analysis will be based on the assumption that the classifier is in some sense optimal, but importantly note that it makes sense to apply the reduction to any off-the-shelf learner. Formally, we say $A$ is {\em optimal} if for every achievable number of positives $P \in \left\{\sum_i c(x_i)~\bigr|~c\in \calc\right\}$, it outputs exactly one classifier that classifies exactly $P$ positives, and this classifier has minimal error among all classifiers which classify exactly $P$ positives. Theorem \ref{th:alg} shows that an optimal classifier can be used to minimize any (monotonic) joint loss 

\begin{theorem}\label{th:alg}
For any monotonic joint loss function $\hat{L}$, any $\calc$, and any optimal learner $A$ for $\calc$, the {\sc Decouple} procedure from Algorithm \ref{alg:decouple} returns a classifier in $\gamma(\calc)$ of minimal joint loss $\hat{L}$. For constant $K$, {\sc Decouple} runs in time linear in the time to run $A$ and polynomial in the number of examples $n$ and time to evaluate $\hat{L}$ and classifiers $c \in \calc$. 
\end{theorem}

\smallskip
\noindent
{\bf Implementation notes.} Note that if the profile or profile is fixed, as in $\hat{L}_{\vec{p}^*}$, then one can simply run the learning algorithm once for each group, targeted at $p^*_k$ positives in each group. Otherwise, also note that to perform the slowest step which involves searching over $O(n^K)$ losses of combinations of classifiers, one can pre-compute the error rates and profiles of each classifier. In the ``big data'' regime of very large $n$, the $O(n^K)$ evaluations of a simple numeric function of profile and losses will not be the rate limiting step.

\begin{algorithm}
\caption{The simple decoupling algorithm partitions data by group and runs the learner on each group. Within each group, the learner outputs one or more classifiers of differing numbers of positives.}\label{alg:decouple}
\begin{algorithmic}[1]
\Procedure{Decouple}{$A, \hat{L}, \langle x_i, y_i\rangle_{i=1}^n, \calx_1, \ldots, \calx_K$}\Comment{Minimize training loss $\hat{L}$ using learner $A$} 
\For{$k=1$ to $K$}
\State $C_k \gets A\bigl(\langle x_i, y_i\rangle_{i:x_i \in \calx_k}\bigr)$\Comment{Learner outputs a set of classifiers}
\EndFor
\State \textbf{return} $\gamma_{\vec{c}}$ that minimizes $\min_{\vec{c} \in C_1\times \ldots \times C_K}\hat{L}\bigl(\langle g_i, y_i, \gamma_{\vec{c}}(x_i)\rangle_{i=1}^n\bigr)$\Comment{$\gamma_{\vec{c}}(x_i)=c_{g_i}(x)$)}
\EndProcedure
\end{algorithmic}
\end{algorithm}

\section{Generalization and transfer learning}\label{sec:transfer}
We now turn to the more general randomized classifier model in which $\calz=[0,1]$ but still with $\caly=\{0,1\}$, and we also consider generalization loss as opposed to simply training loss. We will define loss in terms of the underlying  joint distribution $\cald$ over $\calx\times \caly$ from which training examples are drawn independently. We define the {\em true} error, {\em true profile}, and true probability:
\begin{align*}
\nu_k&=\Pr[x \in \calx_k] =\E[n_k/n]\\
p_k&= \E\bigl[z \I[x \in \calx_k]\bigr] =\E[\hat{p}_k]\\
\err_k &= \E\bigl[|y-z|~|~x \in \calx_k\bigr] =\E[\emp_k | n_k >0]
\end{align*}

Joint loss $L$ is defined on the joint distribution $\mu$ on $g, y, z \in [K] \times \caly \times \calz$ induced by $\cald$ and a classifier $c:\calx \rightarrow \calz$. A distributional joint loss  $L$ is related to empirical joint loss $\hat{L}$ in that $L=\lim_{n \rightarrow \infty}\E[\hat{L}]$, i.e., the limit of the empirical joint loss as the number of training data grows without bound (if it exists). For many of the losses in Table \ref{tab:losses} the limit exists, e.g.,
\begin{align*}
L_1 &= \E[|y-z|]=\sum_{k} \nu_k\err_k \\
L_{NP\lambda}&=\lambda L_1+(1-\lambda)\sum_k \left|\hat{p}_k-\frac{1}{K}\sum_{k'} \hat{p}_{k'}\right|
\end{align*}

Fixing the marginal distribution over $[K] \times \caly$, joint loss $L:[0,1]^{2K} \rightarrow \reals$ can be viewed as a function of $\err_1, p_1, \ldots, \err_K, p_K$ (in addition to group probabilities $\Pr[g(x)=k]$ which are independent of the classification). In addition to requiring monotonicity, namely $L$ being nondecreasing in $\err_k$ fixing all other parameters, we will assume that $L$ is continuous with a bound on the rate of change of the form:
\begin{equation}|L(\err_1, p_1, \ldots, \err_K, p_K)-L(\err_1', p_1', \ldots, \err_K', p_K')| \leq R \sum_k \bigl(\nu_k |\err_k-\err_k'| + |p_k-p_k'|\bigr),\label{eq:continuity}\end{equation}
for parameter $R \geq 0$ and all $\err_k, \err_k', p_k, p_k' \in [0,1]$. Note that the $\nu_k$ in the above bound is necessary for our analysis because a loss that depends on $\err_k$ without $\nu_k$ may require exponentially large quantities of data to estimate and optimize over if $\nu_k$ is exponentially small. Of course, alternatively $\nu_k$ could be removed from this assumption by imposing a lower bound on all $\nu_k$.

Many losses, such as $L_1$ and $L_{NP\lambda}$ above, can be shown to satisfy this continuity requirement for $R=1$ and $R=2$, respectively. We also note that the reduction we present can be modified to address certain discontinuous loss functions. For instance, for a given target allocation (i.e., a fixed fraction of positive classifications for each group), one simply finds the classifier of minimal empirical error for each group which achieves the desired fraction of positives as closely as possible.

A transfer learning algorithm for $\calc$ is $A:(\calx \times\{0,1\})^* \times (\calx \times\{0,1\})^* \rightarrow 2^{\calc}$, where $A$ takes {\em in-group examples} $\langle x_i,y_i\rangle_{i=1}^n$ and {\em out-group examples} $\langle x_i',y_i'\rangle_{i=1}^{n'}$, both from $\calx\times \{0,1\}$. This is also called {\em supervised domain adaptation}. The distribution of out-group examples is different (but related to)  the distribution of in-group samples. The motivation for using the out-group examples is that if one is trying to learn a classifier on a small dataset, accuracy may be increased using related data. 

\begin{algorithm}
\caption{The general decoupling algorithm uses a transfer learning algorithm $T$.}\label{alg:transfer}
\begin{algorithmic}[1]
\Procedure{GeneralDecouple}{$T, \hat{L}, \langle x_i, y_i\rangle_{i=1}^{n}, \calx_1, \ldots, \calx_K$}
\For{$k=1$ to $K$}
\State $n_k \gets |\{i\leq n~|~x_i \in \calx_k\}|$
\State $C_k \gets T\bigl(\langle x_i, y_i\rangle_{i:x_i \in \calx_k}, \langle x_i, y_i\rangle_{i:x_i \not\in \calx_k} \bigr)$\Comment{Run transfer learner, output is a set}
\ForAll{$c \in C_k$}
\State $\displaystyle
\hat{p}_k[c] \gets \frac{1}{n}\sum_{i:x_i \in \calx_k} c(x_i)$\Comment{Estimate profile}
\State $\displaystyle
\emp_k[c] \gets \frac{1}{n_k}\sum_{i:x_i \in \calx_k} |y_i-c(x_i)|$\Comment{Estimate error rates}
\EndFor
\EndFor
\State \textbf{return} $\gamma_{\vec{c}}$ for $\vec{c} \in \argmin_{C_1\times \ldots \times C_K}\hat{L}\left(\emp_1[c_1],\ldots,\emp_K[c_K],\hat{p}_1[c_1],\ldots,\hat{p}_K[c_K]\right)$
\EndProcedure
\end{algorithmic}
\end{algorithm}

In the next section, we describe and analyze a simple transfer learning algorithm that down-weights samples from the out-group. For that algorithm, we show:
\begin{theorem}\label{thm:transfer}
Suppose that, for any two groups $j, k \leq K$ and any classifiers $c, c' \in \calc$, 
\begin{equation}
\label{eqn:delta}
|(\err_j(c)-\err_j(c'))-(\err_k(c)-\err_k(c'))|\leq\Delta
\end{equation}
For algorithm \ref{alg:transfer} with the transfer learning algorithm described in Section \ref{sec:transalg}, with probability $\geq 1-\delta$ over the $n$ iid training data, the algorithm outputs $\hat{c}$ with,
$$L(\hat{c})\leq \min_{c \in \calc}L(c) + 5RK\tau + R\sum_k \min\left(\tau \sqrt{\frac{1}{\nu_k-\tau}}, \Delta \right),$$
where $\tau = \sqrt{\frac{2}{n} \log(8|\calc|(n+K)/\delta)}.$ For constant $K$, the run-time of the algorithm is polynomial in $n$ and the runtime of the optimizer over $\calc$.
\end{theorem}
The assumption in (\ref{eqn:delta}) states that the performance difference between classifiers is similar across different groups and is weaker than an assumption of similar classifier performance across groups. Note that it would follow from a simpler but stronger requirement that $|\err_j(c)-\err_k(c)|\leq \Delta/2$ by the triangle inequality.

Parameter settings (see Lemma \ref{lem:yuck}) and tighter bounds can be found in the next section. However, we can still see qualitatively, that as $n$ grows, the bound decreases roughly like $O(n^{-1/2})$ as expected. We also note that for groups with large $\nu_k$, as we will see in the next section, the transfer learning algorithm places weight 0 on (and hences ignores) the out-group data. For small\footnote{For very small $\nu_k <\tau$, the term $\nu_k -\tau$ is negative (making the left side of the above min imaginary), in which case we define the min to be the real term on the right.} $\nu_k$, the algorithm will place significant weight on the out-group data.


\subsection{A transfer learning algorithm $T$}\label{sec:transalg}

In this section, we describe analyze a simple transfer learning algorithm that down-weights\footnote{If the learning algorithm doesn't support weighting, subsampling can be used instead.} out-group examples by parameter $\theta \in [0,1]$. To choose $\theta$, we can either use cross-validation on an independent held-out set, or $\theta$ can be chosen to minimize a bound as we now describe. The cross-validation, which we do in our experiments, is appropriate when one does not have bounds at hand on the size of set of classifiers or the difference between groups, as we shall assume, or when one simply has a black-box learner that does not perfectly optimize over $\calc$. We now proceed to derive a bound on the error that will yield a parameter choice $\theta$.

Consider $k$ to be fixed. For convenience, we write $n_{-k}=n-n_k$ as the number of samples from other groups.  
Define $\emp_{-k}$ and $\err_{-k}$ analogous to $\emp_k$ and $\err_k$ for out-of-group data $x_i\not\in\calx_k$.

Instead of outputting a set of classifiers, one for each different number of positives within group $k$, it will be simpler to think of the group-$k$ profile $\hat{p}_k=P$ as being specified in advance, and we hence focus our attention on the subset of classifiers, 
$$\calc_{kP} = \left\{c \in \calc~\left|~\frac{1}{n}\right.\sum_{i:x_i \in \calx_k} c(x_i)=P\right\},$$
which depends on the training data. The bounds in this section will be uninteresting, of course, when $\calc_{kP}$ is empty (e.g., in the unlikely event that $x_1=x_2=\ldots=x_n$, the only realizable $\hat{p}_k$ of interest are $0$ and $1$). The general algorithm will simply run the subroutine described in this section $n_k+1\leq n+1$ times, once for each possible value of $\hat{p}_k$.\footnote{In practice, classification learning algorithms generally learn a single real-valued score and consider different score thresholds.} Of course, $|\calc_{kP}|\leq |\calc|$.

As before, we will assume that the underlying learner is optimal, meaning that given a weighted set of examples $(w_1,x_1,y_1),\ldots,(w_n,x_n,y_n)$ with total weight $W=\sum w_i$, it returns a classifier $c \in \calc_{kP}$ that classifier has minimal weighted error $\sum \frac{w_i}{W}|y_i-c(x_i)|$ among all classifiers in $\calc_{kP}$. 

We now derive a closed-form solution for $\theta$, the (approximately) optimal down-weighting of out-group data for our transfer learning algorithm.  This solution depends on a bound $\Delta$ (defined in Theorem \ref{thm:transfer}) on the difference in classifier ranking across different groups.  For small $\Delta$, the difference in error rates of each pair of classifiers is approximately the same for in-group and out-group data.  In this case, we expect generalization to work well and hence $\theta\approx 1$.  For large $\Delta$, out-group data doesn't provide much guidance for the optimal in-group classifier, and we expect $\theta\approx0$.

Finally, for a fixed $k$ and $\theta\in[0,1]$, let $\hat{c}$ be a classifier that minimizes the empirical loss when out-of-group samples are down-weighted by $\theta$, i.e.,
$$\hat{c}\in \argmin_{c\in\calc_{kP}}\,n_k\emp_k(c)+\theta n_{-k}\emp_{-k}(c),$$
and $c^*$ be an optimal classifier that minimizes the true loss, i.e.,
$$c^* \in \argmin_{c\in\calc_{kP}}\err_k(c).$$

We would like to choose $\theta$ such that $\err_k(\hat{c})$ is close to $\err_k(c^*)$.  In order to derive a closed-form solution for $\theta$ in terms of $\Delta$, we use concentration bounds to bound the expected error rates of $\hat{c}$ and $c^*$ in terms of $\Delta$ and $\theta$, and then choose $\theta$ to minimize this expression.

\begin{lemma}
\label{lem:errorratebound}
Fix any $k\leq K, P, n_k, n_{-k} \geq 0$ and $\Delta, \theta \geq 0$. Let $\langle x_i, y_i\rangle_{i=1}^n$ be $n=n_k+n_{-k}$ training examples drawn from $\cald$ conditioned on exactly $n_k$ belonging to group $k$. Let $\hat{c} \in \argmin_{c \in \calc_{kP}}\,n_k\emp_k(c)+\theta n_{-k}\emp_{-k}(c)$ be any minimizer of empirical error when the non-group-$k$ examples have been down-weighted by $\theta$. Then,
$$\Pr\left[\err_k(\hat{c})\leq \min_{c \in \calc_{kP}}\err_k(c) +f(\theta, n_k, n_{-k}, \Delta, \delta)\right] \geq 1-\delta,$$
where the probability is taken over the $n=n_k+n_{-k}$ training iid samples, and $f$ is defined as:
\begin{equation}
\label{eq:f}
f(\theta, n_k, n_{-k}, \Delta, \delta)=\frac{1}{n_k+\theta n_{-k}}\left(\sqrt{2(n_k+\theta^2n_{-k})\log\frac{2|\calc|}{\delta}}+\theta n_{-k}\Delta\right).
\end{equation}
\end{lemma}

Unfortunately, the minimum value of $f$ is a complicated algebraic quantity that is easy to compute but not easy to directly interpret. Instead, we can see that:
\begin{lemma}\label{lem:yuck}
For $f$ from Equation (\ref{eq:f}),
\begin{equation}\label{eq:g}
g(n_k, n_{-k}, \Delta, \delta)=\min_{\theta \in [0,1]} f(\theta, n_k, n_{-k}, \Delta,\delta)\leq \min\left(\sqrt{\frac{2}{n_k}\log\frac{2|\calc|}{\delta}}, \sqrt{\frac{2}{n}\log\frac{2|\calc|}{\delta}}+\frac{n_{-k}}{n}\Delta\right),
\end{equation}
with equality if and only if $n_k\geq \frac{2}{\Delta^2}\log\frac{2|\calc|}{\delta}$ in which case the minimum occurs at $\theta=0$ where $g(n_k, n_{-k}, \Delta)=\sqrt{\frac{2}{n_k}\log\frac{2|\calc|}{\delta}}$. Otherwise the minimum occurs at, 
$$\theta^*=\sqrt{\frac{\beta^2}{4} + \frac{n_{-k}}{n_k}(1-\beta)}-\frac{\beta}{2}\in (0,1),$$ for $\beta=\Delta^2\frac{2}{n_k}\log (2|\calc|/\delta)$.
\end{lemma}
In the above, the two extremes are when $\theta=0$ or $\theta=1$, but as long as $n_k< \frac{2}{\Delta^2}\log\frac{2|\calc|}{\delta}$ the optimal choice of $\theta$ will be strictly in between 0 and 1 and will give a strictly better bound than stated in the lemma above.

\section{Experiments}\label{sec:experiments}
In this section, we describe two experiments, one numerical performed across a number of datasets and the second visual and anecdotal to help illustrate what is happening.

\subsection{Running across multiple datasets}
For this experiment, we used data that is ``semi-synthetic'' in that the 47 datasets are ``real'' (downloaded from \href{http://openml.org}{openml.org}) but an arbitrary binary attribute was used to represent a sensitive attribute, so $K=2$.  The base classifier was chosen to be least-squares linear regression for its simplicity (no parameters), speed, and reproducibility.

In particular, each dataset was a univariate regression problem with balanced loss for squared error, i.e., $\hat{L}_B=\frac{1}{2}(\emp_1+\emp_2)$ where $\emp_k=\sum_{i:g_i=k} (y_i-z_i)^2/n_k$. To gather the datasets, we first selected the problems with twenty or fewer dimensions. Classification problems were converted to regression problems by assigning $y=1$ to the most common class and $y=0$ to all other classes. Regression problems were normalized so that $y \in [0,1]$. Categorical attributes were similarly converted to binary features by assigning 1 to the most frequent category and 0 to others.

\begin{figure}
  \centering\includegraphics[width=6.5in]{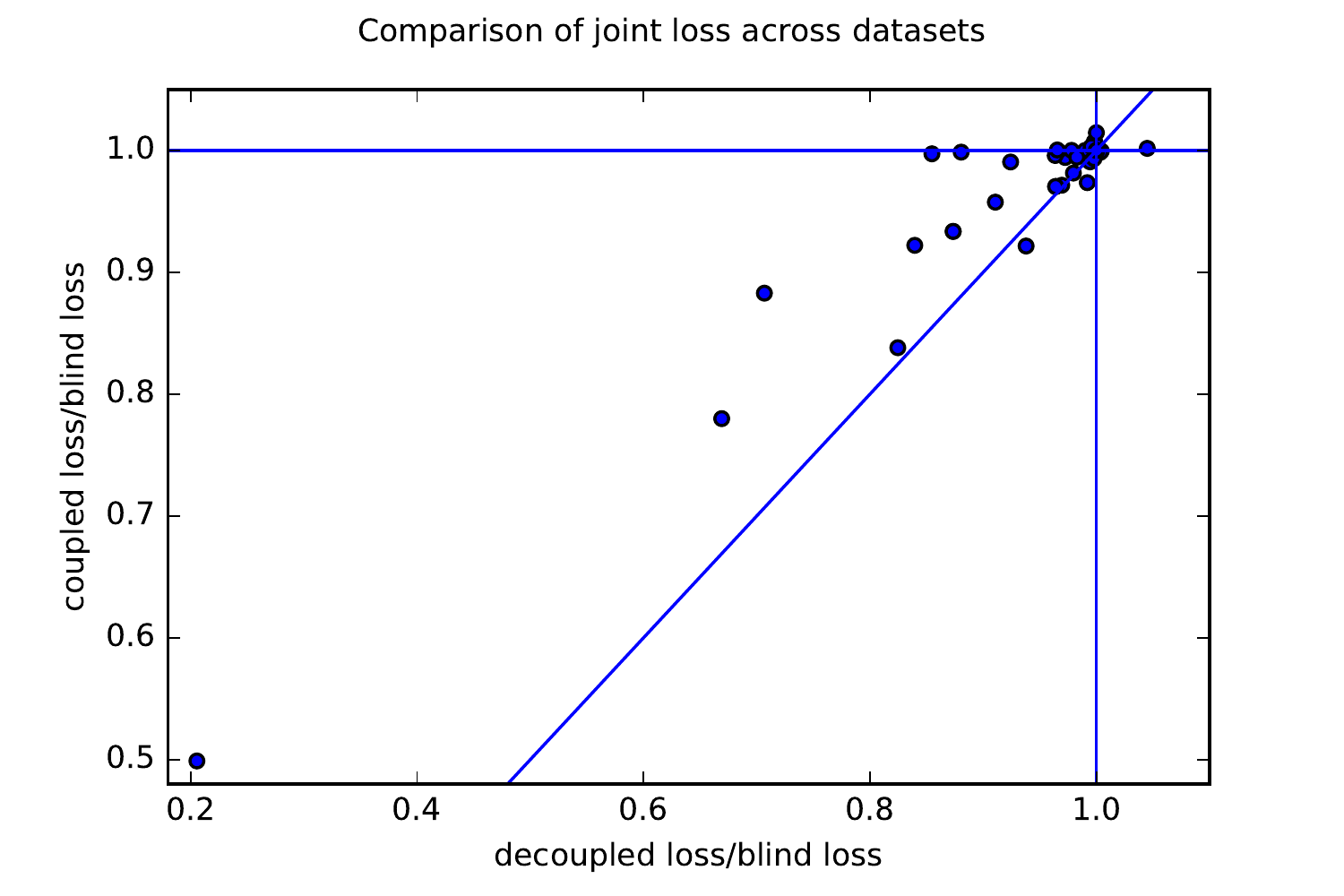}
  \caption{Comparing the joint loss of our decoupled algorithm with the coupled and blind baselines. Each point is a dataset. A ratio less than 1 means that the loss was smaller for the decoupled or coupled algorithm than the blind baseline, i.e, that using the sensitive feature resulted in decreased error. Points above the diagonal represent datasets in which the decoupled algorithm outperformed the coupled one. }
  \label{fig:comparison}
\end{figure}

The sensitive attribute was chosen to be the first binary feature such that there were at least 100 examples in both groups (both 0 and 1 values). Further, large datasets were truncated so that there were at most 10,000 examples in each group. If there was no appropriate sensitive attribute, then the dataset was discarded. We also discarded a small number of ``trivial'' datasets in which the data could be perfectly classified (less than 0.001 error) with linear regression. 
The openml id's and detailed error rates of the 45 remaining datasets are given in Appendix \ref{ap:openml}.  

All experiments were done with five-fold cross-validation to provide an unbiased estimate of generalization error on each dataset. Algorithm \ref{alg:transfer} was implemented, where we further used five-fold cross validation (within each of the outer folds) to choose the best down-weighting parameter $\theta \in \{0,2^{-10}, 2^{-9}, \ldots, 1\}$ for each group. Hence, least-squares regression was run $5*5*11=275$ times on each dataset to implement our algorithm.

\begin{figure}
  \centering\includegraphics[width=6.5in]{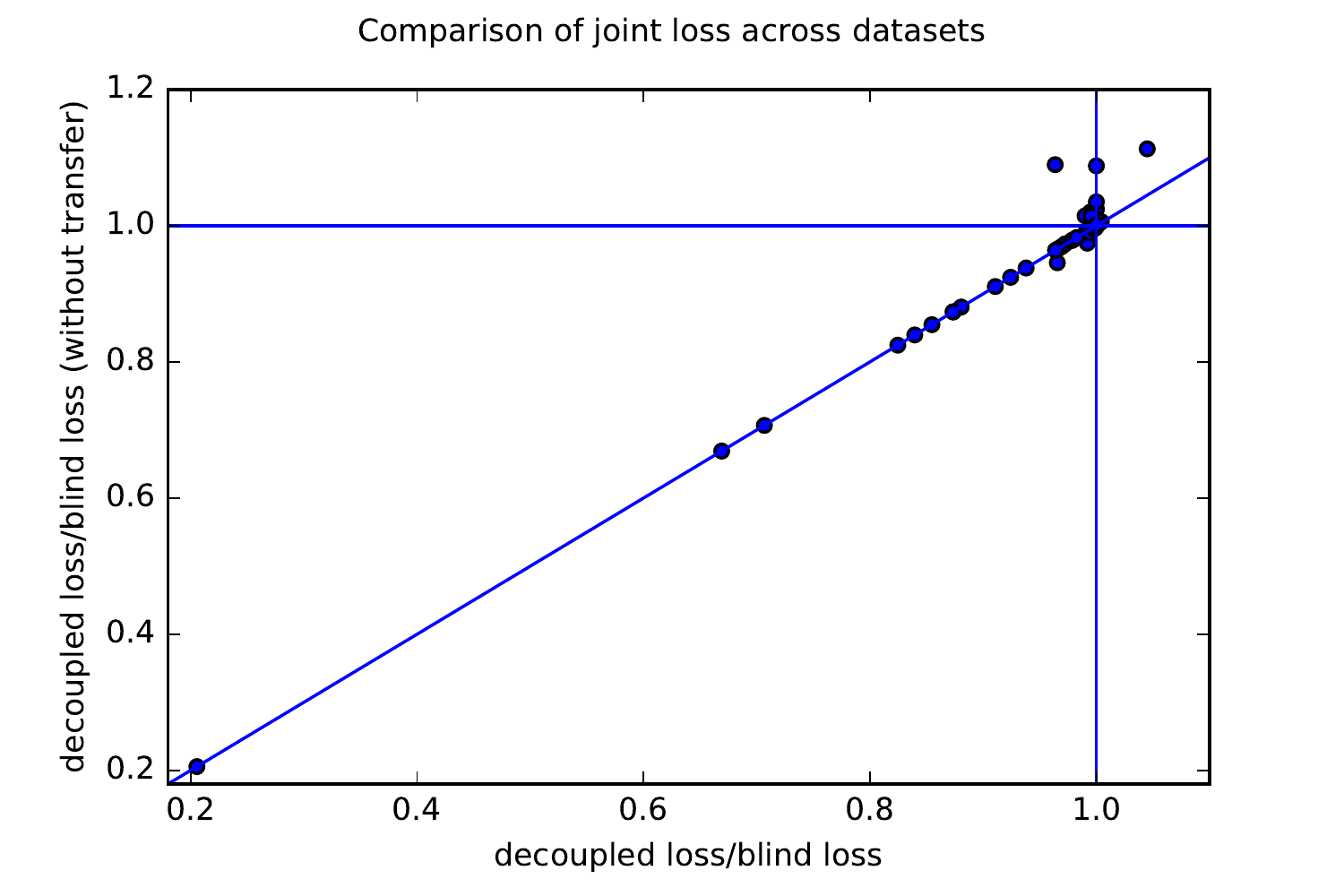}
  \caption{Comparing the joint loss of our decoupled algorithm with the decoupled algorithm with and without transfer learning. Each point is a dataset. A ratio less than 1 means that the loss was smaller for the decoupled algorithm than the blind baseline. Points above the diagonal represent datasets in which transfer learning improved performance compared to decoupling without transfer learning. }
  \label{fig:transcomparison}
\end{figure}

The baselines were considered: the {\em blind} baseline is least-squares linear regression that has no access to the sensitive attribute, the {\em coupled} baseline is least-squares linear regression that can take into account the sensitive attribute.   

Figure \ref{fig:comparison} compares the loss of the coupled baseline (x-axis) and our decoupled algorithm (y-axis) to that of the blind baseline. In particular, the log ratio of the squared errors is plotted, as this quantity is immune to scaling of the $y$ values. Each point is a dataset. Points to the left of 1 ($x<1$) represent datasets where the coupled classifier outperformed the blind one. Points below the horizontal line $y<1$ represent points in which the decoupled algorithm outperformed the indiscriminate baseline. Finally, points below the diagonal line $x=y$ represent datasets where the decoupled classifier outperformed the coupled classifier. 

Figure \ref{fig:transcomparison} compares transfer learning to decoupling without any transfer learning (i.e., just learning on the in-group data or setting $\theta=0$).As one can see, on a number of datasets, transfer learning significantly improves performance. In fact, without transfer learning the coupled classifiers significantly outperform decoupled classifiers on a number datasets.

\subsection{Image retrieval experiment}\label{sec:imageexperiment}

In this section, we describe an anecdotal example that illustrates the type of effect the theory predicts, where a classifier biases towards minority data that which is typical of the majority group. We hypothesized that standard image classifiers for two groups of images would display bias towards the majority group, and that a decoupled classifier could reduce this bias. More specifically, consider the case where we have a set $X=\calx_1 \cup \calx_2$ of images, and want to learn a binary classifier $c:X\rightarrow \{0,1\}$. We hypothesized that a coupled classifier would display a specific form of bias we call {\em majority feature bias}, such that images in the minority group would rank higher if they had features of images in the majority group.

\begin{figure}
  \centering\includegraphics[width=2.5in]{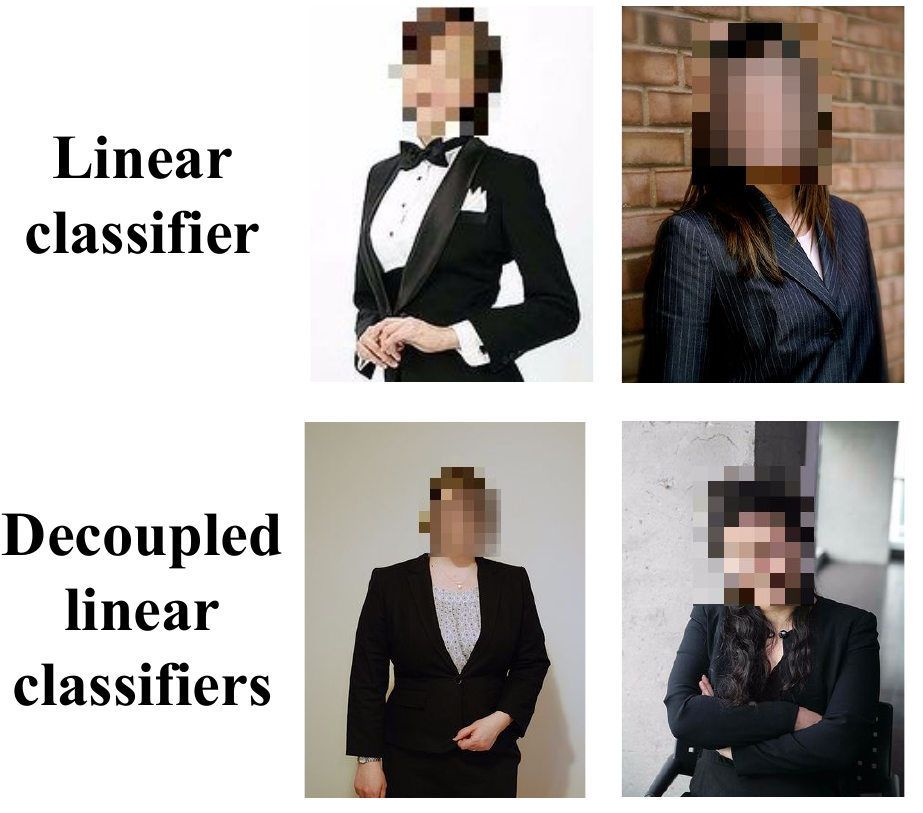}
  \caption{Differences between image classifications of ``suit'' using standard linear classifiers and decoupled classifiers (trained using standard neural network image features). The females selected by the linear classifier are wearing a tuxedo and blazer more typical of the majority male group. }
  \label{fig:image-search}
\end{figure}

We tested this hypothesis by training classifiers to label images as ``suit'' or ``no suit''. We constructed an image dataset by downloading the \href{http://www.image-net.org/synset?wnid=n04350905}{``suit, suit of clothes''} synset as a set of positives, and \href{http://www.image-net.org/synset?wnid=n09624168}{``male person''} and \href{http://www.image-net.org/synset?wnid=n09619168}{``female person''} synsets as the negatives, from ImageNet~\cite{ImageNetCVPR09}. We manually removed images in the negatives that included suits or were otherwise outliers, and manually classified suits as ``male'' or ``female'', removing suit images that were neither. We used the pre-trained \href{https://github.com/BVLC/caffe/tree/master/models/bvlc\_reference\_caffenet}{BVLC CaffeNet model} -- which is similar to the AlexNet mode from~\cite{AlexNet} -- to generate features for the images and clean the dataset. We used the last fully connected of layer (``fc7'') of the CaffeNet model as features, and removed images where the most likely label according to the CaffeNet model was ``envelope'' (indicating that the image was missing), or ``suit, suit of clothes'' or ``bow tie, bow-tie, bowtie'' from the negatives. The dataset included 506 suit images (462 male, 44 female) and 1295 no suit images (633 male, 662 female). 

We then trained a coupled and decoupled standard linear support vector classifier (SVC) on this dataset, and provide anecdotal evidence that the decoupled classifier displays less majority feature bias than the coupled classifier. We trained the coupled SVC on all images, and then ranked images according to the predicted class. We trained decoupled SVCs, with one SVC trained on the male positives and all negatives, and the other on female positives and all negatives. Both classifiers agreed on eight of the top ten ``females'' predicted as ``suit'', and  Fig.~\ref{fig:image-search} shows the four images (two per classifier) that differed. One of the images found by the coupled classifier is a woman in a tuxedo (typically worn by men), which may be an indication of majority feature bias; adding a binary gender attribute to the coupled classifier did not change the top ten predictions for ``female suit.'' We further note that we also tested both the coupled and decoupled classifier on out-of-sample predictions using 5-fold cross-validation, and that both were highly accurate (both had 94.5\% accuracy, with the coupled classifier predicting one additional true positive).

We emphasize that we present this experiment to provide an anecdotal example of the potential advantages of a decoupled classifier, and we do not make any claims on generalizability or effect size on this or other real world datasets because of the small sample size and the several manual decisions we made.

\section{Conclusions}

In this paper, we give a simple technical approach for a practitioner using ML to incorporate sensitive attributes. Our approach avoids unnecessary accuracy tradeoffs between groups and can accommodate an application-specific objective, generalizing the standard ML notion of loss. For a certain family of ``weakly monotonic'' fairness objectives, we give a black-box {\em reduction} that can use any off-the-shelf classifier to efficiently optimize the objective. In contrast to much prior work on ML which first requires complete fairness, this work requires the application designer to pin down a specific loss function that trades off accuracy for fairness.

Experiments demonstrate that decoupling can reduce the loss on some datasets for some potentially sensitive features. 

\bibliography{decouple}

\appendix

\section{Proofs}

\begin{proof}[Proof of Theorem \ref{thm:tradeoff}]
For linear separators, consider the parity function on the last two bits over the uniform distribution on $\calx$, i.e., $y = \chi_{S}(x)$ for $S=\{d-1,d\}$ \citep[see, e.g.,][]{BFJ+94}. This can be perfectly classified by a decoupled linear classifier because it can output a different function based on the last bit. The squared error of any single (coupled) linear function is at least 1/4 \ since $|S|>1$ \citep[again see][]{BFJ+94}. Moreover, the constant predictor $1/2$ achieves squared error 1/4 for any distribution, hence, the cost of coupling is exactly 1/4. 

For linear separators over the same four examples, it is easy to see that no linear separator correctly classifies all 4 cases of the last two bits and hence has loss at least 1/4, while a decoupled separator can perfectly classify the data. Hence, the cost of coupling of linear separators is at least 1/4 as well.

For decision trees of size $\leq 2^s$, consider the parity function on the last $s+1$ bits, again with the uniform distribution over all $2^d$ examples in $\calx$. This parity function is perfectly computable by a size $2^s$ decoupled classifier since a decision tree with $2^s$ leaves can perfectly compute a parity on $s$ bits. 

Now, the error rate of a decision tree is the weighted average of the error rates of its leaves, where each leaf is weighted by the proportion of examples in that leaf. For the uniform distribution, the weight is simply $2^{-\text{depth}}$. Moreover, it is easy to see that any leaf whose path does not involve all $s+1$ relevant bits will have error rate 1/2. Finally, it is also easy to see that the total weight of leaves with depth $\leq s$ is at least $1/2$, hence the error rate of tee with $\leq 2^s$ leaves is at least $1/4$.
\end{proof}

\begin{proof}[Proof of Lemma \ref{lem:monotonicity}]
Fix $\langle g_i,y_i\rangle_{i=1}^n\in ([K] \times \caly)^*$ and profile $\vec{p}$, the latter of which also determines the number of positive classifications in each group -- what remains is to decide which examples are positively classified. Consider swapping the classifications of any two examples $y_i,z_i$ and $y_j,z_j$. If $y_i=y_j$ or $z_i=z_j$, then such a swap has no effect on the error or false positive/negative rates, and hence does not change the loss under definitions \ref{def:monotonicity1} or \ref{def:monotonicity3}. It also can be seen to hold for Definition \ref{def:monotonicity2}, since it can be applied to the swap in either way.  
Now, consider $y_i=z_j\neq y_j=z_i$. Clearly the swap simultaneously decreases the error count by 2 and decreases both the false positive and false negative counts by 1 each. By all three definitions the loss cannot increase.   

Hence, we see that the loss is determined by the numbers of false and true positives and negatives, and the loss decreases in these quantities similarly by all three definitions. Of course, the joint loss need only be defined integer numbers of false and true positives and negatives. However, among these numbers it must be monotonic in $\emp_k$ because one can move amongst these numbers by simply swapping labels on one false positive and one false negative. Further, it is easy to see that any function defined on a subset of $\reals$ that is nondecreasing can be extended to a nondecreasing function on $\reals$ by interpolation.
\end{proof}

\begin{proof}[Proof of Observation \ref{obs:kevefef}]
We use Definition \ref{def:monotonicity2} of monotonicity. A swap in which $y_i=y_j$ or $z_i=z_j$ neither changes $\emp_k$ nor $\FP_k$, so it has no effect on the loss. When $y_i=z_j\neq y_j=z_i$, swapping $z_i$ and $z_j$ removes one false positive and false negative and hence decreases $\emp_k$ by $2/n_k$ and decreases the first term in the loss by $2(1-\lambda)/n_k$. It also decreases $\FN_k$ and $\FP_k$ by $1/n_k$ which in turn can increase the second term in the loss by at most $2 \lambda/n_k$. Such an increase will in fact occur, if, e.g., the error rate on the other group is 1, e.g., if $z_i=1-y_i$ for all examples in that group. Hence, such a swap will decrease the loss, as required by Definition \ref{def:monotonicity2}, if and only if $\lambda/n_k \leq (1-\lambda)/n_k$, i.e., iff $\lambda \leq 1/2$.  
\end{proof}

\begin{proof}[Proof of Theorem \ref{th:alg}]
Let $\vec{c}^*=(c_1^*,\ldots,c_K^*)$ be a vector of classifiers that minimizes $$L\bigl(\langle g_i,y_i, \gamma_{\vec{c}^*}(x_i)\rangle_{i=1}^n\bigr)=L\bigl(\langle g_i,y_i, c^*_{g_i}(x_i)\rangle_{i=1}^n\bigr).$$ We argue that {\sc Decouple} returns a classifier of no greater loss. In particular, let $\vec{p}^*$ be the profile of $\vec{c}^*$. By the optimality of $A$, for each $k\leq K$ it finds a $c_k \in \calc$ which classifies exactly $n_kp^*_k$ positives and has error $\emp_k(c)\leq \emp_k(c^*_k)$. By the monotonicity of $L$, the joint loss of $\vec{c}=(c_1,\ldots,c_k)$ is no greater than that of $c^*$, since it has the same profile $\vec{p}^*$. Since {\sc Decouple} returns the classifier of minimal loss amongst the candidates, its loss is no greater than that of $c^*$.

In terms of runtime, first note that we run $A$ exactly $K$ times. The output of each call to $A$ is of size at most $n+1$, since by assumption there is at most one classifier with each number of positive classifications. Hence the product space has size at most $(n+1)^K=\text{poly}(n)$, each of which involves $n$ classifier evaluations and 1 loss evaluation, which implies the theorem. 
\end{proof}

\begin{proof}[Proof of Lemma \ref{lem:errorratebound}]
	By the Hoeffding bound, for independent random variables $X_1,\ldots,X_n$, where $X_i\in[a_i,b_i]$, 
	$$\Pr\left[\left|\sum_{i=1}^nX_i-E[\sum_{i=1}^nX_i]\right|\geq t\right]\leq2\exp\left(-\frac{2t^2}{\sum_{i=1}^n(b_i-a_i)^2}\right).$$
	For the random training set of size $n$, let $X_i=|c(x_i)-y_i|$ for $x_i\in \calx_k$ and $X_i=\theta|c(x_i)-y_i|$ for $x_i\not\in \calx_k$.  Then these are independent random variables with $X_i\in[0,1]$ for $x_i\in\calx_k$ and $X_i\in[0,\theta]$ for $x_i\not\in\calx_k$. Furthermore, the sum of these random variables is $n_k\emp_k(c)+\theta n_{-k}\emp_{-k}(c)$ and their expectation is $n_k \err_k(c)+\theta n_{-k} \err_{-k}(c)$.  Applying the Hoeffding bound, the probability this sum and expectation differ by more than $t$ is at most $2\exp\left(-\frac{2t^2}{n_k+n_{-k}\theta^2}\right)$.  Picking $t=\sqrt{\frac{1}{2}(n_k+\theta^2n_{-k})\log(\frac{2|\calc|}{\delta})}$ such that this quantity is at most $\delta/|\calc|$, and taking the union bound over $c\in\calc\supseteq \calc_{kP}$, we have that,
$$n_k \err_k(\hat{c}) + \theta n_{-k}\err_{-k}(\hat{c}) \leq n_k \err_k(c^*) + \theta n_{-k}\err_{-k}(c^*) +2t,$$
with probability $\geq 1-\delta$.
Finally, from the assumption on $\Delta$ in Theorem \ref{thm:transfer}, it follows that,
$$|\err_k(\hat{c})-\err_k(c^*) - (\err_{-k}(\hat{c}) - \err_{-k}(c^*))| \leq \Delta,$$
which combined with the above implies that with probability $\geq 1-\delta$:
$$(n_k + \theta n_{-k})(\err_k(\hat{c})-\err_k(c^*)) \leq 2t + \theta n_{-k}\Delta,$$
which completes the proof.
\end{proof}

\begin{proof}[Proof of Lemma \ref{lem:yuck}]
Let $r=\sqrt{\frac{2}{n_k}\log \frac{2|\calc|}{\delta}}$ and $z=\theta n_{-k}/n_k$. Then $f$ can be rewritten as:
$$f(\theta, n_k, n_{-k}, \Delta)=\frac{1}{1+z}\left(r\sqrt{1+z^2\frac{n_k}{n_{-k}}}+z\Delta\right).$$
Taking the derivative with respect to $z$ gives:
$$\frac{\partial}{\partial z}\left(\frac{1}{1+z}\left(r\sqrt{1+z^2\frac{n_k}{n_{-k}}}+z\Delta\right)\right) = \frac{1}{(1+z)^2}\left(\Delta-r\frac{1-\frac{n_k}{n_{-k}}z}{\sqrt{1+\frac{n_k}{n_{-k}}z^2}}\right)$$
We see that the quantity is non-negative for $z=0$  iff $\Delta\geq r$ (i.e., $n_k\geq \frac{2}{\Delta^2}\log\frac{2|\calc|}{\delta}$). In this case, we also see that it is non-negative for $z>0$ hence the minimum occurs at $\theta=0$, as claimed in the theorem. Otherwise the derivative at 0 is negative, and setting the derivative equal to 0 gives the $\theta^*>0$  specified in the theorem.

The two terms in the minimum of the theorem simply correspond to $f$ at $\theta=0$ and $\theta = 1$, which of course are larger than the minimum of $f$ over all $\theta\geq 0$.
\end{proof}

Using these lemmas, we can now prove Theorem \ref{thm:transfer}.
\begin{proof}[Proof of Theorem \ref{thm:transfer}]
By Hoeffding bounds, we have that with probability $\geq 1-\delta/4$, all $K$ counts will be close to their expectations among the first $n$ training examples:
$$\Pr\left[\max_k |n_k -\nu_k n| \leq \tau n\right]\geq 1-\frac{\delta}{4}.$$
Also by Hoeffding bounds, we have that with probability $\geq 1-\delta/4$, for all $c\in \calc$ and all $k \leq K$ their empirical error rates are simultaneously close to the their expectations:
$$|\emp_k(c)-\err_k(c)|\nu_k\leq \sqrt{\frac{1}{2n}\log\frac{8|\calc|K}{\delta}}.$$
Similarly by Hoeffding bounds, we have that with probability $\geq 1-\delta/4$, for all $c\in \calc$ and all $k \leq K$ all profiles are simultaneously close to the their expectations:
$$|\hat{p}_k(c)-p_k(c)|\leq \sqrt{\frac{1}{2n}\log\frac{8|\calc|K}{\delta}}.$$
If the latter two of the above likely events happen, then we have that
\begin{equation}\label{eq:genz}
|L(c)-\hat{L}(c)| \leq KR\sqrt{\frac{2}{n}\log \frac{8|\calc|K}{\delta}}< KR\tau\text{ for all }c \in \gamma(\calc),
\end{equation}
by the assumption on continuity with bound $R$, and for the $\tau$ stated in the theorem, since $\hat{L}(c)=L(\emp_1, \hat{p}_1,\ldots, \emp_K, \hat{p}_K)$ while $L(c)=L(\err_1, p_1,\ldots, \err_K, p_K)$

Next, imagine choosing the first $n$ training examples by first choosing the counts $n_k$ and then choosing the examples conditional on those counts. Once the counts have been chosen, there are $n_k+1$ possible values for each $\hat{p}_k$, so there are at most $n+K$ different classifiers output by the transfer learning algorithm over all groups.

Substituting $\frac{\delta}{4(n+K)}$ for $\delta$ in Lemmas \ref{lem:errorratebound} and \ref{lem:yuck}, we have that with probability $\geq 1-\delta/4$, simultaneously for each group $k$ and for every realizable target $P$, the most accurate down-weighted classifier (for the choice of $\theta$ from Lemma \ref{lem:yuck}) will achieve error within $g(n_k, n_{-k}, \Delta, \delta/(4(n+K)))$ of the most accurate classifier in $\calc_{kP}$. 

Now, let $c^*\in \argmin_{c \in \calc}L(c)$ be a classifier minimizing $L$ on $\mu$, and let $(\hat{p}^*_1, \ldots, \hat{p}^*_K)$ be its empirical profile. Then for each $k$, with the target $P=\hat{p}^*_k$, the transfer learner outputs a classifier $\hat{c}_k$ with $\err_k(\hat{c}_k)\leq \err_k(c^*) + g(n_k, n_{-k}, \Delta, \delta/(4(n+K)))$. Now, if the first event happens, namely that all the $n_k$ are within $\tau n$ of their expectations, then we have a bound on $L(\hat{c}_1, \ldots, \hat{c}_K)$ of
$$L(c^*) + R \sum_k  \min\left(\sqrt{\frac{2}{\nu_kn-\tau n}\log\frac{8|\calc|(n+K)}{\delta}}, \sqrt{\frac{2}{n}\log\frac{8|\calc|(n+K)}{\delta}}+\frac{n-\nu_kn+\tau n}{n}\Delta\right).$$
Combining this with Equation (\ref{eq:genz}) gives the following bound, since if the difference between true and estimated error is $a$, the empirical best will have error within $2a$ of the true best:
$$L(\hat{c})\leq \min_{c \in \calc}L(c) + 2RK\tau + R\sum_k \min\left(\tau \sqrt{\frac{1}{\nu_k-\tau}}, \tau+\left(1-\nu_k + \tau\right)\Delta \right),$$
Since the failure probabilities of any of the four conditions was $\delta/4$, the stated bound holds with probability $\geq 1-\delta$. The bound in the theorem follows from the fact that $\Delta<2$ and $\nu_k \geq 0$, hence,
$$R\sum_k \min\left(\tau \sqrt{\frac{1}{\nu_k-\tau}}, \tau+\left(1-\nu_k + \tau\right)\Delta \right) \leq 3R\tau K + \min\left(\tau \sqrt{\frac{1}{\nu_k-\tau}}, \Delta \right)$$
\end{proof}

\section{Dataset ids}\label{ap:openml}
For reproducibility, the id's and feature names for the 47 open ml datasets were as follows:
(21, 'buying'),
 (23, 'Wifes\_education'),
 (26, 'parents'),
 (31, 'checking\_status'),
 (50, 'top-left-square'),
 (151, 'day'),
 (155, 's1'),
 (183, 'Sex'),
 (184, 'white\_king\_row'),
 (292, 'Y'),
 (333, 'class'),
 (334, 'class'),
 (335, 'class'),
 (351, 'Y'),
 (354, 'Y'),
 (375, 'speaker'),
 (469, 'DMFT.Begin'),
 (475, 'Time\_of\_survey'),
 (679, 'sleep\_state'),
 (720, 'Sex'),
 (741, 'sleep\_state'),
 (825, 'RAD'),
 (826, 'Occasion'),
 (872, 'RAD'),
 (881, 'x3'),
 (915, 'SMOKSTAT'),
 (923, 'isns'),
 (934, 'family\_structure'),
 (959, 'parents'),
 (983, 'Wifes\_education'),
 (991, 'buying'),
 (1014, 'DMFT.Begin'),
 (1169, 'Airline'),
 (1216, 'click'),
 (1217, 'click'),
 (1218, 'click'),
 (1235, 'elevel'),
 (1236, 'size'),
 (1237, 'size'),
 (1470, 'V2'),
 (1481, 'V3'),
 (1483, 'V1'),
 (1498, 'V5'),
 (1557, 'V1'),
 (1568, 'V1'),
 (4135, 'RESOURCE'),
 (4552, 'V1')

\end{document}